\documentclass{article}

\usepackage{preprint_arxiv}

\usepackage[utf8]{inputenc} % allow utf-8 input
\usepackage[T1]{fontenc}    % use 8-bit T1 fonts
\usepackage{hyperref}       % hyperlinks
\usepackage{url}            % simple URL typesetting
\usepackage{booktabs}       % professional-quality tables
\usepackage{amsfonts}       % blackboard math symbols
\usepackage{nicefrac}       % compact symbols for 1/2, etc.
\usepackage{microtype}      % microtypography
\usepackage{xcolor}

\usepackage{times}
\usepackage{mathtools}
\usepackage{multirow}
\usepackage{graphicx}
\DeclareGraphicsExtensions{.pdf,.png,.jpg}
\graphicspath{{figs/}}
\usepackage{amsthm}
\usepackage{amssymb}
\usepackage{bbm}
\usepackage{bm}
\usepackage{commath}

\usepackage[title,titletoc]{appendix}

\usepackage{algorithm}
\usepackage{algorithmic}

\input{mymacros.tex}
\newcommand{\Ind}[1]{\ensuremath{\mathbbm{1} \hspace{-0.03in} \left[#1\right]}}                     % Indicator function.
\newcommand{\InNorm}[1]{{\left\vert\kern-0.2ex\left\vert\kern-0.2ex\left\vert #1 
    \right\vert\kern-0.2ex\right\vert\kern-0.2ex\right\vert}}                    % Induced Norm.

\newcommand{\InNormII}[1]{{\left\vert\kern-0.2ex\left\vert\kern-0.2ex\left\vert #1 
    \right\vert\kern-0.2ex\right\vert\kern-0.2ex\right\vert}_2}                    % Induced 2 Norm (Spectral Norm).

\newcommand{\InNormInfty}[1]{{\left\vert\kern-0.2ex\left\vert\kern-0.2ex\left\vert #1 
    \right\vert\kern-0.2ex\right\vert\kern-0.2ex\right\vert}_{\infty}}           % Induced Infinity norm.

\newcommand{\Abs}[1]{\ensuremath{\lvert #1 \rvert}}                              % Absolute value.
\newcommand{\AAbs}[1]{\ensuremath{\left \lvert #1 \right \rvert}}                              % Absolute value auto braces.
\newcommand{\Inner}[2]{\langle #1, #2 \rangle}                    % Inner product

\newtheorem{definition}{Definition}
\newtheorem{proposition}{Proposition}

\newtheorem{lemma}{Lemma}

\newtheorem{theorem}{Theorem}
\newtheorem{remark}{Remark}

\newcommand{\subto}{\mathrm{subject\ to}}

\def\1{\bm{1}}

\def\vzero{{\bm{0}}}
\def\vone{{\bm{1}}}
\def\vmu{{\bm{\mu}}}

\def\va{{\bm{a}}}

\def\vu{{\bm{u}}}
\def\vv{{\bm{v}}}

\def\vy{{\bm{y}}}

\def\eva{{a}}

\def\evy{{y}}

\def\mA{{\bm{A}}}

\def\mD{{\bm{D}}}

\def\mL{{\bm{L}}}
\def\mM{{\bm{M}}}

\def\mS{{\bm{S}}}
\def\mT{{\bm{T}}}

\def\mV{{\bm{V}}}
\def\mW{{\bm{W}}}
\def\mX{{\bm{X}}}
\def\mY{{\bm{Y}}}

\def\mLambda{{\bm{\Lambda}}}

\DeclareMathAlphabet{\mathsfit}{\encodingdefault}{\sfdefault}{m}{sl}
\SetMathAlphabet{\mathsfit}{bold}{\encodingdefault}{\sfdefault}{bx}{n}

\def\gC{{\mathcal{C}}}

\def\gG{{\mathcal{G}}}

\def\gJ{{\mathcal{J}}}
\def\gK{{\mathcal{K}}}

\def\gO{{\mathcal{O}}}

\def\sP{{\mathbb{P}}}

\def\sR{{\mathbb{R}}}

\def\emA{{A}}

\def\emT{{T}}

\def\emY{{Y}}

\newcommand{\tand}{\mathrm{and}}

\newcommand{\E}{\mathbb{E}}

\DeclareMathOperator{\sign}{sign}
\DeclareMathOperator{\Tr}{Tr}

\allowdisplaybreaks 

\title{A Thorough View of Exact Inference in Graphs from the Degree-4 Sum-of-Squares Hierarchy}

\author{
Kevin Bello\footnotemark[2]
\qquad
Chuyang Ke\footnotemark[2]
\qquad 
Jean Honorio\footnotemark[2]\\
\footnotemark[2]\ \  Department of Computer Science, Purdue University, West Lafayette, IN, USA.
}

\begin{document}
\maketitle

\begin{abstract}
	Performing inference in graphs is a common task within several machine learning problems, e.g., image segmentation, community detection, among others. 
	For a given undirected connected graph, we tackle the statistical problem of \textit{exactly} recovering an unknown ground-truth binary labeling of the nodes from a \textit{single} corrupted observation of each edge.
	Such problem can be formulated as a quadratic combinatorial optimization problem over the boolean hypercube, where it has been shown before that one can (with high probability and in polynomial time) exactly recover the ground-truth labeling of graphs that have an isoperimetric number that grows with respect to the number of nodes (e.g., complete graphs, regular expanders).
	In this work, we apply a powerful hierarchy of relaxations, known as the sum-of-squares (SoS) hierarchy, to the combinatorial problem.
	Motivated by empirical evidence on the improvement in exact recoverability, we center our attention on the degree-4 SoS relaxation and set out to understand the origin of such improvement from a graph theoretical perspective.
	We show that the solution of the dual of the relaxed problem is related to finding edge weights of the Johnson and Kneser graphs, where the weights fulfill the SoS constraints and intuitively allow the input graph to increase its algebraic connectivity.
	Finally, as byproduct of our analysis, we derive a novel Cheeger-type lower bound for the algebraic connectivity of graphs with \textit{signed} edge weights.
\end{abstract}

\section{Introduction}
	Inference in graphs spans several domains such as social networks, natural language processing, computational biology, computer vision, among others. 
	For example, let $\mX \in \{-1,0,+1\}^{n \times n}$ be some \textit{noisy} observation, e.g., a social network (represented by a graph), where the output is a labeling $\vy \in \{-1,+1\}^n$ of the nodes, e.g., an assignment of each individual to a cluster.
	In the example, for the entries of $\mX$, a value of $0$ means no interaction (no edge) between two individuals (nodes), a value of $+1$ can represent an agreement of two individuals, while a value of $-1$ can represent disagreement.
	One can then predict a labeling $\vy$ by solving the following quadratic form over the hypercube $\{\pm 1\}^n$,
	\begin{align}
		\max_{\vy \in \{\pm 1\}^n} \vy^\top \mX \vy. \label{eq:initial_opt}		
	\end{align}
	The above formulation is a well-studied problem that arises in multiple contexts, including Ising models from statistical physics \citep{barahona1982computational}, finding the maximum cut of a graph \citep{goemans1995improved}, the Grothendieck problem \citep{grothendieck1956resume, khot2011grothendieck}, stochastic block models \citep{abbe2016exact}, and structured prediction problems \citep{globerson2015hard}, to name a few.
	However, the optimization problem above is NP-hard in general and only some cases are known to be exactly solvable in polynomial time.
	For instance, \citet{chandrasekaran2008complexity} showed that it can be solved exactly in polynomial time for a graph with low treewidth via the junction tree algorithm; \citet{schraudolph2009efficient} showed that the inference problem can also be solved exactly in polynomial time for planar graphs via perfect matchings; while \citet{boykov2006graph} showed that \eqref{eq:initial_opt} can be solved exactly in polynomial time via graph cuts for binary labels and sub-modular pairwise potentials.
			
	In this work, we consider a generative model proposed by \citet{globerson2015hard} in the context of structured prediction, and study which conditions on the graph allow for exact recovery (inference).
	In this case, following up with the example on social networks above, each individual can have an opinion labeled by $-1$ or $+1$, and for each pair of individuals that are \textit{connected}, we observe a \textit{single} measurement of whether or not they have an agreement in opinion, but the value of each measurement is flipped with probability $p$. 
	Since problem \eqref{eq:initial_opt} is a hard computational problem, it is common to relax the problem to a convex one.
	In particular, \citep{abbe2016exact,amini2018semidefinite,bello2019exact,bello2020fairness} studied the sufficient conditions for a semidefinite programming relaxation (SDP) of problem \eqref{eq:initial_opt} to achieve exact recovery.
	In contrast to those works, we will focus on the sum-of-squares (SoS) hierarchy of relaxations \citep{parrilo2000structured,lasserre2001global,barak2014sum}, which is a sequential tightening of convex relaxations based on SDP.
	We study the SoS hierarchy because it is tighter than other known hierarchies such as the Sherali-Adams and Lov{\'a}sz-Schrijver hierarchies~\citep{laurent2003comparison}.
	In addition, our motivation to study the level-2 or degree-4 SoS relaxation stems from three reasons.
	First, higher-levels of the hierarchy, while polynomial time solvable, are already computationally very costly.
	This is one of the reasons the SoS hierarchy have been mostly used as a proof system for finding lower bounds in hard problems (e.g., for the planted clique problem, see \citep{meka2015sum}).
	Second, little is still known about the level-2 SoS relaxation, where \cite{bandeira2018gramian}  and \cite{cifuentes2020geometry} are attempts to understand its geometry. 
	Third, there is empirical evidence on the improvement in exact recoverability with respect to SDP, an example of which is depicted in Figure \ref{fig:intro}.
	\begin{figure}[!tb]
		\centering	
		\includegraphics[width=0.4\linewidth]{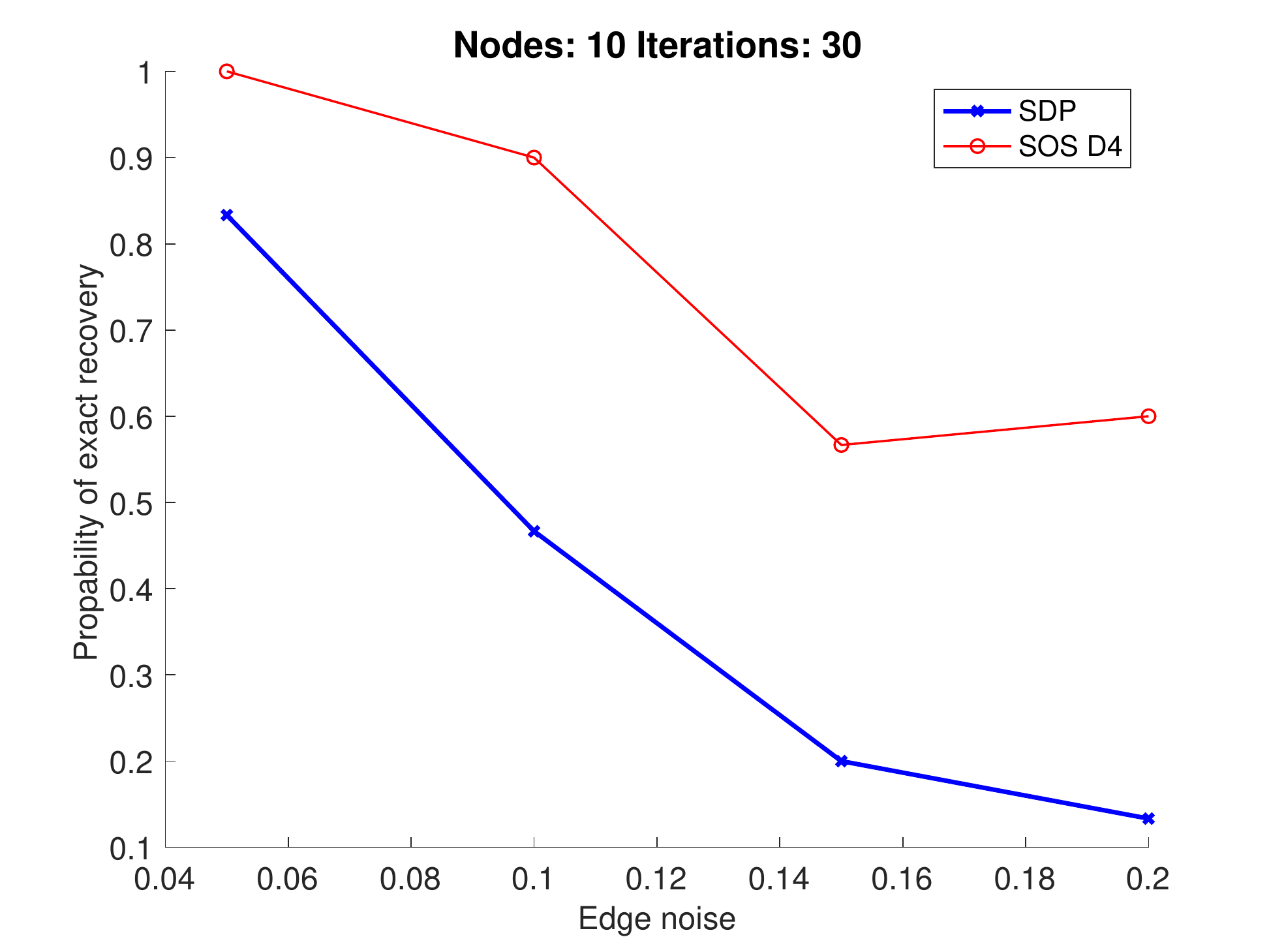}
		\caption{
		A comparison between the degree-4 SoS and SDP relaxations in the context of structured prediction. 
		We observe that SoS attains a higher probability of exact recovery, for different levels of edge noise $p$. 
		(See Section \ref{sec:prelim} for a formal problem definition.)
		}
		\label{fig:intro}
	\end{figure}			
	
	\paragraph{Contributions.}
	While it is known that the level-2 SoS relaxation has a tighter search space than that of SDP,
	it is not obvious why it can perform better than SDP for \textit{exact recovery}.
	In this work, we aim to understand the origin of such improvement from a graph theoretical perspective.
	We show that the solution of the dual of the relaxed problem is related to finding edge weights of the Johnson and Kneser graphs, where the weights fulfill the SoS constraints and intuitively allow the input graph to increase its algebraic connectivity.
	Finally, as byproduct of our analysis, we derive a novel Cheeger-type lower bound for the algebraic connectivity of graphs with \textit{signed} edge weights.
	
	We emphasize that the objective of this work is on the \emph{understanding} of exact recoverability by using the degree-4 SoS.
	Scalability of the SoS hierarchy is an important open problem that is actively under study \cite{weisser2016bounded,Erdogdu2017} and is beyond the scope of our work.
	
\section{Preliminaries}
	\label{sec:prelim}
	This section introduces the notation used throughout the paper and formally defines the problem under analysis.
	
	Vectors and matrices are denoted by lowercase and uppercase bold faced letters respectively (e.g., $\va,\mA$), while scalars are in normal font weight (e.g., $a$).
	For a vector $\va$, and a matrix $\mA$, their entries are denoted by $\eva_i$ and $\emA_{i,j}$ respectively.
	Indexing starts at $1$, with $\mA_{i,:}$ and $\mA_{:,i}$ indicating the $i$-th row and $i$-th column of  $\mA$ respectively.
	The eigenvalues of a $n \times n$ matrix $\mA$ are denoted as $\lambda_{i}(\mA)$, where $\lambda_1$ and $\lambda_n$ correspond to the minimum and maximum eigenvalue respectively.
	Finally, the set of integers $\{1, \ldots, n\}$ is represented as $[n]$.
	
	\paragraph{Problem definition.} 
	We aim to predict a vector of $n$ node labels $\vy = (y_1, \dots, y_n)^\top$, where $y_i \in \{+1,-1\}$, from a set of observations $\mX$, where $\mX$ corresponds to noisy measurements of edges.
	These observations are assumed to be generated from a ground truth labeling $\vys$ by a generative process defined via an undirected connected graph $G = (V, E)$, where $V = [n]$, and an edge noise $p \in (0, 0.5)$.
	For each edge $(u,v) \in E$, we have a \textit{single} independent edge observation $X_{u,v} = \ys_u \ys_v$ with probability $1-p$, and $X_{u,v} = -\ys_u \ys_v$ with probability $p$.
	While for each edge $(u,v) \notin E$, the observation $X_{u,v}$ is always $0$.
	Thus, we have a \textit{known} undirected connected graph $G$, an \textit{unknown} ground truth label vector $\vys \in \{+1,-1\}^n$, noisy observations $\mX \in \{-1,0,+1\}^{n\times n}$.
	Given that we consider only edge observations, our goal is to understand when one can predict, in polynomial time and with high probability, a vector label $\vy \in \{-1,+1\}^n$ such that $\vy \in \{\vys,-\vys\}$.

	Given the aforementioned generative process, our focus will be to solve the following optimization problem, which stems from using maximum likelihood estimation \citep{globerson2015hard}:
	\begin{align}
		\quad \max_{\vy }  \quad &\vy^\top \mX \vy, \ \ \subto \ \ y_i = \pm 1,\ \forall i \in [n]. \label{eq:opt_dis}
	\end{align} 
	In general, the above combinatorial problem is NP-hard to compute, e.g., see results on grids by \citet{barahona1982computational}.
	Let $\vydis$  denote the optimizer of eq.\eqref{eq:opt_dis}.
	It is clear that for any label vector $\vy$, the negative label vector $-\vy$ attains the same objective value in eq.\eqref{eq:opt_dis}. 
	Thus, we say that one can achieve exact recovery by solving eq.\eqref{eq:opt_dis} if $\vydis \in \{\vys, -\vys\}$.
	Given the computational hardness of solving eq.\eqref{eq:opt_dis}, in the next subsections we will revise approaches that relax problem \eqref{eq:opt_dis} to one that can be solved in polynomial time.
	Then, our focus will be to understand the effects of the structural properties of the graph $G$ in achieving, with high probability, exact recovery in the continuous problem.
	
	\subsection{Semidefinite Programming Relaxation}
		A popular approach for approximating problem \eqref{eq:opt_dis} is to consider a larger search space that is simpler to describe and is convex.
		In particular, let $\mY = \vy \vy^\top$, that is, $Y_{i,j} = y_i y_j$ and noting that $\mY$ is a rank-1 positive semidefinite matrix.
		We can rewrite the objective of problem \eqref{eq:opt_dis} in matrix terms as follows, $\vy^\top \mX \vy = \Tr (\mX \vy \vy^\top) = \Tr (\mX \mY) = \Inner{\mX}{\mY}$.
		Thus, we have
		\begin{align}
			\max_{\mY}  \ \ \Inner{\mX}{\mY}, \ \
			\subto \ \ \mY \succeq 0,\ \ Y_{i,i} = 1, \forall i \in [n]. \label{eq:opt_sdp}
		\end{align}
		Let $\mYsdp$  denote the optimizer of the problem above, then, in this case, we say that exact recovery is realized by solving eq.\eqref{eq:opt_sdp} if $\mYsdp = \vys \vys^\top$. 
		The only constraint dropped in problem \eqref{eq:opt_sdp} with respect to problem \eqref{eq:opt_dis} is the rank-1 constraint, which makes problem \eqref{eq:opt_sdp} convex.
		The above relaxation is known as semidefinite programming (SDP) relaxation and is typically used as an approximation algorithm.
		That is, after obtaining a continuous solution $\mYsdp \in \Rnn$, a rounding procedure is performed to recover an approximate solution in $\{\pm 1\}^n$, e.g., see~\citep{goemans1995improved,nesterov1998semidefinite}.
		However, SDP relaxations have also been analyzed for exact inference, for instance, \cite{abbe2016exact} and \cite{amini2018semidefinite} studied exact recovery in the context of stochastic block models, while, \cite{bello2019exact,bello2020fairness} studied exact recovery in the context of structured prediction.
		
		In the next subsection, we will introduce tighter levels of relaxations known as the SoS hierarchy, and we will see that it turns out that SDP relaxations correspond to the first level of the SoS hierarchy.
		
	\subsection{Sum-of-Squares Hierarchy}
		We start this section by introducing additional notation for describing the SoS hierarchy.
		Let $[n]^{\leq d} = \{\emptyset\} \cup [n]^1 \cup \ldots \cup [n]^d$ denote the set of (possibly empty) tuples, of length up to $d$, composed of the integers from $1$ to $n$, e.g., $[2]^{\leq 2} = \{\emptyset, (1), (2), (1,1), (1,2), (2,1), (2,2)\}$.
		Also, let the summation between two tuples be the concatenation of all the elements in them, e.g., for $\gC_1 = (1,1,2), \gC_2 = (3,1)$ we have $\gC_1 + \gC_2 = (1,1,2,3,1)$.
		We use $\psi(\gC)$ to denote the tuple with elements from $\gC$ sorted in ascending order, e.g., for $\gC = (2,1,1,3)$ we have $\psi(\gC) = (1,1,2,3)$.
		We also use $\Abs{\gC}$ to denote the cardinality of $\gC$. 
		For two distinct tuples $\gC_1$ and $\gC_2$, the expression $\gC_1 < \gC_2$ means that either $\Abs{\gC_1} < \Abs{\gC_2}$, or $\Abs{\gC_1} = \Abs{\gC_2}$ and $\exists i$ such that the $i$-ith entry of $\gC_2$ is greater than the $i$-th entry of $\gC_1$.
		Then, for a set of tuples $\mathfrak{C} = \{\gC_1, \ldots, \gC_k\}$, we say that $\mathfrak{C}$ is in lexicographical order if $\gC_i < \gC_j$ for all $i < j$.		
		Finally, for a matrix $\mY \in \sR^{[n]^{\leq \ell} \times [n]^{\leq \ell}}$, we index its rows and columns by using tuples in $[n]^{\leq \ell}$ ordered lexicographically, e.g., for $\mY \in \sR^{[5]^{\leq 3}\times [5]^{\leq 3}}$ we have that $\mY_{(1,1,2), (5)}$ corresponds to the entry at row $(1,1,2)$ and column $(5)$.
		
		It is convenient to rewrite the objective of problem \eqref{eq:opt_dis} as a polynomial optimization problem, i.e., $\sum_i \sum_j X_{i,j} y_i y_j$, so that the standard machinery of SoS optimization \citep{lasserre2001global,parrilo2000structured,laurent2009sums} can be applied to formulate the degree-$d$ relaxation.
		Then, for an even number $d$, the degree-$d$ (or level $\nicefrac{d}{2}$) SoS relaxation of problem \eqref{eq:opt_dis} takes the form
		\begin{align}
			\max_{\mY \in \sR^{[n]^{\leq \frac{d}{2}} \times [n]^{\leq \frac{d}{2}}}} \   &\sum_{i=1}^{n} \sum_{j=1}^{n} X_{i,j} Y_{(i),(j)},\label{eq:opt_sos_degree_d} \\
			\subto \ 
			&\mY \succeq 0; 
			\ \ \mY_{(\emptyset)(\emptyset)} = 1;
			\ \ \mY_{(i)+\gC_1, (i) + \gC_2} = \mY_{\gC_1, \gC_2}, \ \forall i \in [n],\ \Abs{\gC_1}, \Abs{\gC_2} \leq \nicefrac{d}{2}-1; \notag \\
			&\mY_{\gC_1,\gC_2} = \mY_{\gC_1', \gC_2'}, \ \forall \psi(\gC_1 + \gC_2) = \psi(\gC_1' + \gC_2'), \ \Abs{\gC_1}, \Abs{\gC_2}, \Abs{\gC_1'}, \Abs{\gC_2'} \leq \nicefrac{d}{2}.  \notag
		\end{align}
		In the problem above, each entry of the matrix $\mY$ corresponds to a reparametrization that takes the form 
		$\mY_{\gC_1, \gC_2} = \prod_{i \in \gC_1} y_i \prod_{j \in \gC_2} y_j = \prod_{i \in \gC_1 + \gC_2} y_i,$
		which is also known as a pseudomoment matrix~\citep{lasserre2001global, laurent2009sums}.
		In problem \eqref{eq:opt_sos_degree_d}, the second constraint can be thought as a normalization constraint.
		The third list of constraints corresponds to 
		$ \prod_{j \in \gC}  y_j\cdot y_i^2 = \prod_{j \in \gC}  y_j, \forall  \Abs{\gC} \leq d - 2,$
		which is equivalent to $y_i = \pm 1$ in problem \eqref{eq:opt_dis}.
		Finally, the last list of constraints corresponds to 
		$ \prod_{j \in \gC_1 + \gC_2} y_j = \prod_{j \in \gC_1' + \gC_2'} y_j, 
		  \forall  \psi(\gC_1 + \gC_2) = \psi(\gC_1' + \gC_2'), \text{ and }  \Abs{\gC_1}, \Abs{\gC_2}, \Abs{\gC_1'}, \Abs{\gC_2'} \leq \nicefrac{d}{2},$
		which states that $\mY_{\gC_1, \gC_2}$ should be invariant to all permutations of the tuple $\gC_1 + \gC_2$.
		One can note that, for $d=2$, the degree-$2$ (or level $1$) SoS relaxation is equivalent to the SDP relaxation in eq.\eqref{eq:opt_sdp}.
		It is clear that for a larger $d$, the degree-$d$ SoS relaxation gives a tighter convex relaxation of problem \eqref{eq:opt_dis}.
		While one can solve problem \eqref{eq:opt_sos_degree_d} to a fixed accuracy using general-purpose SDP algorithms in polynomial time in $n$,
		the computational complexity will be of order $n^{\gO(d)}$.
		Thus, it is important that $d$ be of low order. 
		
		In the next section, we center our attention to the degree-$4$ SoS relaxation and in understanding how it can help improving the exact recovery rate with respect to the SDP (or degree-$2$ SoS) relaxation.
				
\section{On Exact Recovery from the Degree-4 SoS Hierarchy}
	As the focus of this section will be on the degree-4 SoS relaxation, we start by formulating the corresponding optimization problem.
	In problem \eqref{eq:opt_sos_degree_d}, for $d=4$, the matrix $\mY$ is in $\sR^{[n]^{\leq 2} \times [n]^{\leq 2}}$, that is, $\mY$ is a matrix of dimension $(1 + n + n^2) \times (1+n+n^2)$.
	\citet[Appendix A]{bandeira2018gramian} showed that one can write an equivalent formulation by using only the principal submatrix of $\mY$ indexed by $[n]^2 \times [n]^2$ (i.e., a matrix of dimension $n^2 \times n^2$).
	The reduced formulation takes the form:
	\begin{align}
		\max_{\mY \in \sR^{[n]^{2} \times [n]^{2}}} \ &\sum_{i=1}^{n} \sum_{j=1}^{n} X_{i,j} Y_{(1,1),(i,j)},\label{eq:opt_sos_degree_4} \\
		\subto \ &\mY \succeq 0;
		\ \ \mY_{(i,i)(j,j)} = 1,\ \forall i,j \in [n];
		\ \ \mY_{(i,i)(j,k)} = \mY_{(i',i')(j,k)},\ \forall  i,i',j,k \in [n]; \notag \\
		&\mY_{(i,j)(k,\ell)} = \mY_{(\pi_1,\pi_2)(\pi_3,\pi_4)}, \forall\ i,j,k,\ell \in [n],\ \pi \in \Pi(i,j,k,\ell),  \notag
	\end{align}
	where $\Pi(i,j,k,\ell)$ is the set of all permutations of $(i,j,k,\ell)$.
	We will go one step further in the reduction and show that one can indeed cast an equivalent formulation to problem \eqref{eq:opt_sos_degree_4} by using only the principal submatrix of $\mY \in \sR^{[n]^2 \times [n]^2}$ indexed by ${[n] \choose 2} \times {[n] \choose 2}$, i.e., a matrix of dimension $\frac{n(n-1)}{2} \times \frac{n(n-1)}{2}$.
	Here, it will be more convenient to use sets instead of tuples for indexing the rows and columns of $\mY$, where ${[n] \choose 2}$ denotes the set of all unordered combinations of length $2$ from the numbers in $[n]$, e.g., ${[3] \choose 2} = \{\{1, 2\}, \{1, 3\}, \{2, 3\}\}$. 
	For further distinction against the matrix $\mY \in \sR^{[n]^{2} \times [n]^{2}}$, we will use $\mYt$ to denote the matrix indexed by ${[n] \choose 2} \times {[n] \choose 2}$.
	
	We will also make use of the next set of definitions, which are important for stating our results.
	\begin{definition}[The level-2 vector]
	\label{def:level2-vector}
	For any vector $\vv \in \sR^n$, its level-2 vector, denoted by $\vvt \in \sR^{{n \choose 2}}$ and indexed by ${[n] \choose 2}$, is defined as $v^{(2)}_{\{i,j\}} = v_i v_j$.	
	\end{definition}
	We also define the level-2 version of a graph as follows.	
	\begin{definition}[The level-2 graph]
	\label{def:level2-graph}
		Let $G = (V,E)$, where $V=[n]$, be any undirected graph of $n$ nodes with adjacency matrix $\mA \in \{0,1\}^{n \times n}$. 
		The level-2 graph of $G$, denoted by $G^{(2)} = ({[n] \choose 2}, E^{(2)})$ and with adjacency matrix $\mAt \in \{0,1\}^{ {n \choose 2} \times {n \choose 2} }$, has its adjacency matrix defined as $\At_{\{i,k\},\{k,j\}} = 1$ if $(i,j) \in E$ for all $i < j < k\in [n]$, and $\At_{\{i,j\},\{k,\ell\}} = 0$ for all $i < j < k < \ell \in [n].$
	\end{definition}
	
	The next type of graphs have been studied for several years within the graph theory community and we will later show how they relate to the solution of the level-2 SoS relaxation.
	\begin{definition}[Johnson graph \citep{holton1993petersen}]
	\label{def:johnson-graph}
		For a set $[n]$, the Johnson graph $\gJ(n,k)$ has all the $k$-element subsets of $[n]$ as vertices, and two vertices are adjacent if and only if the intersection of the two vertices (subsets) contains $(k-1)$-elements.
	\end{definition}
	
	\begin{definition}[Kneser graph \citep{lovasz1978kneser}]
	\label{def:kneser-graph}
		For a set $[n]$, the Kneser graph $\gK(n,k)$ has all the $k$-element subsets of $[n]$ as vertices, and two vertices are adjacent if and only if the two vertices (subsets) are disjoint.
	\end{definition}
	
	From Definitions \ref{def:johnson-graph} and \ref{def:kneser-graph}, we are interested in $\gJ(n,2)$ and $\gK(n,2)$, where we first note that $\gK(n,2)$ is the complement of $\gJ(n,2)$.
	We also note that for a graph $G$ of $n$ nodes, by construction, the level-2 graph of $G$ is always a subgraph of the Johnson graph $\gJ(n,2)$, and is equal to $\gJ(n,2)$ if and only if $G$ is the complete graph of $n$ nodes.
	Finally, since our observation matrix $\mX$ depends on a graph $G$, one can also extend $\mX$ to a matrix in $\{-1, 0,+1\}^{{n \choose 2} \times {n \choose 2}}$.
	We will use $\mXt$ to denote the level-2 version of $\mX$.
	Specifically, $\Xt_{\{i,k\},\{k,j\}} = X_{i,j}$ for all $i < j < k \in [n]$, and $\Xt_{\{i,j\},\{k,\ell\}} = 0$ for all $i < j < k < \ell \in [n].$
	For further clarity, we illustrate the level-2 construction of $\mX$ in Figure \ref{fig:level2_example}, where the input graph is a 2 by 2 grid.
	\begin{figure}[!tb]
		\centering	
		\includegraphics[width=0.5\linewidth]{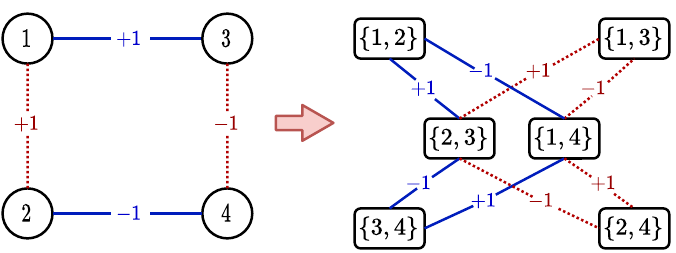}
		\caption{
			Illustration of the level-2 construction of $\mX$. 
			The edge values in the grid graph correspond to the observation $\mX$, while the edge values on the right graph correspond to level-2 matrix $\mXt.$
			The solid blue and dotted red lines indicate that the observation is correct and corrupted, respectively.
		}
		\label{fig:level2_example}
	\end{figure}	
	
	Next, we present an optimization problem that is equivalent to problem~\eqref{eq:opt_sos_degree_4} but in terms of the level-2 constructions defined above.
	For notational convenience, we will use $\mS^{+(ijk\ell)}_{-(stuv)}$ to denote a sparse \textit{symmetric} matrix such that the only non-zero entries are $\mS_{\{i,j\},\{k,\ell\}} = 1$, and $\mS_{\{s,t\},\{u,v\}} = -1$.
	\begin{align}
		\max_{\mYt \in \sR^{{[n] \choose 2} \times {[n] \choose 2}}} \ &\frac{1}{n-2} \Inner{\mXt}{\mYt}, \label{eq:opt_sos} \\
		\subto \
		&\mYt \succeq 0;
		\  \mYt_{\gC,\gC} = 1, \forall \gC \in {\textstyle {[n] \choose 2}};
		\  \Inner{\mS^{+(ikkj)}_{-(ik'k'j)}}{\mYt} = 0, \forall  i < j < k < k' \in [n]; \notag \\
		&\Inner{\mS^{+(ijk\ell)}_{-\pi(ijk\ell)}}{\mYt} = 0,  \forall i < j < k < l \in [n],\ \pi \in \Pi(i,j,k,l).  \notag
	\end{align}
	\begin{proposition}
	\label{prop:equivalent_form}
		Problem \eqref{eq:opt_sos} is equivalent to problem \eqref{eq:opt_sos_degree_4}.
	\end{proposition}
	All proofs are detailed in Appendix \ref{app:proofs}.
	\begin{remark}
		Let $\vyt$ be the level-2 vector of the \textit{ground-truth} labeling $\vys$, and let $\mYt^*$ be the optimizer of problem \eqref{eq:opt_sos}.
		Then, we say that exact recovery is realized if $\mYt^* = \vyt {\vyt}^\top$. 
	\end{remark}

	\subsection{The Dual Problem}
	\label{sec:dual_problem}
		A key ingredient for our analysis is the dual formulation of problem \eqref{eq:opt_sos}, which takes the following form
		\begin{align}
			\min_{\mVt,\ \vmu} \ &\Tr(\mVt), \label{eq:dual_problem}\\
			\subto \ &\mVt \mathrm{\ is\ diagonal}, \notag \\
			&\mLambdat = \mVt - \frac{\mXt}{n-2} + \sum_{i < j < k < k'} \hspace{-0.1in} \mu^{ikkj}_{ik'k'j} \mS^{+(ikkj)}_{-(ik'k'j)}  + \sum_{ \substack{i < j < k < \ell \\ \pi \in \Pi(i,j,k,\ell)}} \hspace{-0.1in} \mu^{ijk\ell}_{\pi(ijk\ell)} \mS^{+(ijk\ell)}_{-\pi(ijk\ell)} \succeq 0, \notag 
		\end{align}
		where $\mVt \in \sR^{{[n] \choose 2} \times {[n] \choose 2}}$, $\mu^{ikkj}_{ik'k'j} \in \sR$ and $\mu^{ijk\ell}_{\pi(ijk\ell)} \in \sR$ are the dual variables of the second constraint, and the third and fourth list of constraints from the primal formulation \eqref{eq:opt_sos}, respectively.
		The dual variable $\vmu$ denotes all the scalars $\mu^{ikkj}_{ik'k'j}$ and $\mu^{ijk\ell}_{\pi(ijk\ell)}$.
		
		We have that if there exists $\mYt, \mVt, \vmu$ that satisfy the Karush-Kuhn-Tucker (KKT) conditions \citep{boyd2004convex}, then $\mYt$ and $\mVt, \vmu$ are primal and dual optimal, and strong duality holds in this case.
		Let $\vyt$ be the level-2 vector of the \textit{ground-truth} labeling $\vys$.
		Since we are interested in exact recovery, we will consider the solution $\mYt = \vyt {\vyt}^\top$ for the rest of our analysis, where it is clear that such setting satisfies the primal constraints.
		Let
		\begin{align}
			\hspace{-0.05in} \mVt = \frac{\diag\big( \mXt \mYt \big)}{n-2} - \diag\Big(  \hspace{-0.08in} \sum_{i < j < k < k'} \hspace{-0.13in}\mu^{ikkj}_{ik'k'j} \mS^{+(ikkj)}_{-(ik'k'j)} \mYt\Big) - \diag\Big( \hspace{-0.1in} \sum_{ \substack{i < j < k < \ell \\ \pi \in \Pi(ijk\ell)}} \hspace{-0.13in} \mu^{ijk\ell}_{\pi(ijk\ell)} \mS^{+(ijk\ell)}_{-\pi(ijk\ell)} \mYt \Big), \label{eq:V_setting}
		\end{align}
		where, for a matrix $\mM$, $\diag(\mM)$ denotes the diagonal matrix formed from the diagonal entries of $\mM$.
		Complementary slackness and stationarity require the trace of $\mVt$ to be equal to the trace of the r.h.s.~of eq.\eqref{eq:V_setting}, which is clearly satisfied by construction.
		Thus, if we find an assignment of $\vmu$ such that $\mLambdat \succeq 0$, we would have an optimal solution since all KKT conditions are fulfilled.
		Nevertheless, we are also interested in $\mYt = \vyt {\vyt}^\top$ being the \textit{unique} optimal solution, where we note that having $\lambda_2(\mLambdat) > 0$ suffices to guarantee a unique solution.
		The argument follows from the fact that, by the setting of eq.\eqref{eq:V_setting}, we have $\mLambdat \vyt = 0$.
		Thus, if $\lambda_2(\mLambdat) > 0$ then $\vyt$ spans all of the null-space of $\mLambdat$.
		Combined with the KKT conditions, we have that $\mYt$ should be a multiple of $\vyt {\vyt}^\top$.
		Since $\mYt$ has diagonal entries equal to 1, we must have that $\mYt = \vyt {\vyt}^\top$.
		
		Putting all pieces together, we have that under eq.\eqref{eq:V_setting}, if for some $\vmu$ we have that $\mLambdat \succeq 0$ and $\lambda_2(\mLambdat) > 0$, then the optimizer of problem \eqref{eq:opt_sos} is $\vyt {\vyt}^\top$, i.e., we obtain exact recovery.
		Since $\vyt$ is an eigenvector of $\mLambdat$ with eigenvalue zero, we focus on controlling the quantity $\lambda_2(\mLambdat) = \min_{\vv \perp \vyt} \frac{\vv^\top \mLambdat \vv}{\vv^\top \vv}$.
		\footnote{This expression comes from the variational characterization of eigenvalues.}
		Also, as $\mLambdat$ depends on the noisy observation $\mXt$, we have that $\mLambdat$ is a random quantity. 
		Then, by using Weyl's theorem on eigenvalues, we have
		\begin{align}
			\lambda_2(\mLambdat) = \lambda_2(\mLambdat - \E[\mLambdat] + \E[\mLambdat]) \geq \lambda_2(\E[\mLambdat]) + \lambda_1(\mLambdat - \E[\mLambdat]). \label{eq:exp_plus_concentration}
		\end{align}
		In eq.\eqref{eq:exp_plus_concentration}, let $t$ be a lower bound to $\lambda_2(\E[\mLambdat])$, i.e., $\lambda_2(\E[\mLambdat]) \geq t$.
		Then, the second summand can be lower bounded by using matrix concentration inequalities.
		Specifically, by using matrix Bernstein inequality \citep{tropp2012user}, one can obtain that $\sP[\lambda_1(\mLambdat - \E[\mLambdat]) \leq -t] \leq \gO(n^2 \exp{-t})$.
		Thus, we can now focus on the first summand, which will be lower bounded by a novel Cheeger-type inequality.
		In the next subsections, we look at the expected value of $\mLambdat$ in more detail.
	
	\subsection{The Relation between \texorpdfstring{$\E[\mLambdat]$}{E[\backslash Lambda]} and the Algebraic Connectivity of the Level-2 Graph}
		In this section, we will show how $\E[\mLambdat]$ is related to the Laplacian matrix of $\Gt$ (the level-2 version of $G$).
		To do so, we will use the following definitions and notation.
		
		For a \textit{signed} weighted graph $H=(U,F)$, we use $\mW^H$ to  denote its weight matrix, that is, the entry $W^H_{i,j} \in \sR$ is the weight of edge $(i,j) \in F$ and is zero if $(i,j) \notin F.$
		For any set $T \subset U$, its boundary is defined as $\bound{T} = \{(i,j) \mid i \in T\ \tand\ j \notin T \}$; while its boundary weight is defined as $\boundW{T} = \sum_{i\in T, j\notin T} W^{H}_{i,j}$.
		The number of nodes in $T$ is denoted by $\abs{T}$.
		The degree of a node is defined as $\deg(i) = \sum_{j\neq i} W^H_{i,j}$.
		\begin{definition}
			Let $H$ be a graph with degree matrix $\mD^{H}$ and weight matrix $\mW^H$, where $\mD^H$ is a diagonal matrix such that $D_{i,i} = \deg(i)$.
			The Laplacian matrix of $H$ is defined as $\mL^H = \mD^H - \mW^H$.
		\end{definition}
		\begin{definition}[Cheeger constant \citep{cheeger1969lower}]
			For a graph $H = (U,F)$ of $n$ nodes, its Cheeger constant is defined as 
			$\phi(H)  = \min_{T\subset U,\ \abs{T} \leq \nicefrac{n}{2}} \frac{\boundW{T}}{\abs{T}}.$
		\end{definition}
		
		\begin{remark}
			For unweighted graphs, the definitions above match the standard definitions for node degree, boundary of a set, and Laplacian matrix.
		\end{remark}
		
		Next, we analyze the scenario where all the scalar dual variables in $\vmu$ are zero, we defer the case when they are not for the next subsection.
		
		\textbf{The $\vmu=\vzero$ scenario.}
		From eq.\eqref{eq:V_setting} we have that $\mVt = \diag( \mXt \mYt )/(n-2)$. 
		Hence, for all $i < j \in [n]$, we have $\E[\Vt_{\{i,j\},\{i,j\}}] = \nicefrac{(1-2p)}{(n-2)}\cdot \deg(\{i,j\})$.
		In addition, we have $\E[\Xt_{\{i,k\},\{k,j\}}] = (1-2p) \ \ys_i \ \ys_j \Ind{(i,j)\in E}$, for all $i < j < k \in [n]$.
		\footnote{
			Recall that if $(i,j) \in E$, then $X_{i,j} = -\ys_i \ys_j$ with probability $p$, and $X_{i,j} = \ys_i \ys_j$ otherwise.
			If $(i,j) \notin E$ then $X_{i,j} = 0$.
		}
		Finally, since $\vmu = \vzero$, we have	$\mLambdat = \mVt - \frac{\mXt}{n-2}$.
		Therefore, 
		\begin{align}
			\E[\mLambdat] = \frac{1-2p}{n-2} \widetilde{\mUpsilon} \mL^{\Gt} \widetilde{\mUpsilon}, \label{eq:upsilon}
		\end{align}
		where $\widetilde{\mUpsilon}$ is a diagonal matrix with entries equal to the entries in $\vyt$.
		Recall that $\yt_{\{i,j\}} = \ys_i \ys_j$ and $\ys_i \in \{\pm 1\}$ for all $i \in [n]$.
		Then, we have that $\widetilde{\mUpsilon}^{-1} = \widetilde{\mUpsilon}$ and, thus, the matrix $\E[\mLambdat]$ and $\frac{1-2p}{n-2} \mL^{\Gt}$ are similar.
		The latter means that both matrices share the same spectrum, i.e., 
		\begin{align}
			\lambda_2(\E[\mLambdat]) = \frac{1-2p}{n-2} \lambda_2(\mL^{\Gt}).\label{eq:similar_matrix}
		\end{align}
		Notice that the level-2 graph $\Gt$ is unweighted since $G$ is unweighted.
		That implies that one can lower bound $\lambda_2(\E[\mLambdat])$ by using existing lower bounds for the second eigenvalue\footnote{The second eigenvalue of the Laplacian matrix is also known as the algebraic connectivity.}  of the Laplacian matrix of $\Gt$.
		In particular, one can have \citep{mohar1991laplacian}
		\begin{align}
			\lambda_2(\E[\mLambdat]) = \frac{1-2p}{n-2} \lambda_2(\mL^{\Gt}) \geq \frac{(1-2p)\phi(\Gt)^2}{2(n-2)\deg_{\max}}. \label{eq:lb_mu_zero}
		\end{align}
		Finally, we note that considering $\vmu = \vzero$ is equivalent to not having the third and fourth list of constraints in problem \eqref{eq:opt_sos}.
		At this point, the reader might wonder if, setting $\vmu = \vzero$ and solving problem \eqref{eq:opt_sos} yields in any better chances of exact recovery than solving problem \eqref{eq:opt_sdp}.
		We answer the latter in the negative.
		\begin{proposition}
		\label{prop:sdp_sos_equal}
			 Without the third and fourth list of constraints, problem \eqref{eq:opt_sos} does not improve exact recoverability with respect to problem \eqref{eq:opt_sdp}.
		\end{proposition}
		The purpose of Proposition \ref{prop:sdp_sos_equal} is to highlight the role that a $\vmu \neq \vzero$ will play in showing the improvement in exact recoverability of the degree-4 SoS relaxation with respect to the SDP relaxation, which is discussed next.

	\subsection{Connections to Systems of Sets and a Novel Cheeger-Type Lower Bound}
	\label{sec:johnson_kneser}
		We start this section by showing how the third and fourth list of constraints of problem \eqref{eq:opt_sos} relate to finding edge weights of the Johnson and Kneser graphs, respectively, so that the Laplacian matrix of a \textit{new} graph is positive semidefinite (PSD).
		
		Note that the third and fourth list of constraints in the SoS relaxation \eqref{eq:opt_sos} do not depend on the input graph, nor on the edge observations or the ground-truth node labels.
		Instead, they are constraints coming from the SoS relaxation, as explained in the subsequent paragraphs to problem \eqref{eq:opt_sos_degree_d}.
		That means that they depend only on the number of nodes, $n$, and on the degree of the relaxation, $d=4$.
		We will illustrate in detail the case of $n=4$ as it is easier to generalize from there to any value of $n$.
		
		Recall that $\mS^{+(ijk\ell)}_{-(stuv)}$ is a symmetric matrix that has non-zero entries $\mS_{\{i,j\},\{k,\ell\}} = 1$, and $\mS_{\{s,t\},\{u,v\}} = -1$.
		By taking advantage of the implicit symmetry constraint from $\mYt \succeq 0$, for $n=4$, one can realize that the third list of constraints in problem \eqref{eq:opt_sos} has six different constraints in total (with their respective dual variables), which are:
		\begin{align*}
			\mu^{1224}_{1334}: \Inner{\mS^{+(1224)}_{-(1334)}}{\mYt} = 0,\quad \mu^{2113}_{2443}: \Inner{\mS^{+(2113)}_{-(2443)}}{\mYt} = 0,\quad \mu^{1332}_{1442}: \Inner{\mS^{+(1332)}_{-(1442)}}{\mYt} = 0,\\
			\mu^{3114}_{3224}: \Inner{\mS^{+(3114)}_{-(3224)}}{\mYt} = 0,\quad \mu^{1223}_{1443}: \Inner{\mS^{+(1223)}_{-(1443)}}{\mYt} = 0,\quad \mu^{2334}_{2114}: \Inner{\mS^{+(2334)}_{-(2114)}}{\mYt} = 0.
		\end{align*}
		Similarly, from the fourth list of constraints we have:
		\begin{align*}
			\mu^{1324}_{1234}: \Inner{\mS^{+(1324)}_{-(1234)}}{\mYt} = 0,\quad \mu^{2314}_{1234}: \Inner{\mS^{+(2314)}_{-(1234)}}{\mYt} = 0.
		\end{align*}
		In the dual formulation \eqref{eq:dual_problem}, for both lists above, the matrices $\mS$ are weighted by the dual variables $\mu$.
		Then, the two weighted summations can be thought of as weight matrices of some graphs.
		Interestingly, such graphs happen to be the Johnson and Kneser graphs
		\footnote{For any $n$, whenever we write the Johnson and Kneser graphs, we refer to $\gJ(n,2)$ and $\gK(n,2)$, respectively.}
		 for the first and second list of constraints above, respectively.
		In Figure \ref{fig:johnson_kneser_example}, we show an illustration of the Johnson and Kneser graphs with edge weights corresponding to the dual variables.
		\begin{figure}[!tb]
			\centering	
			\includegraphics[width=0.5\linewidth]{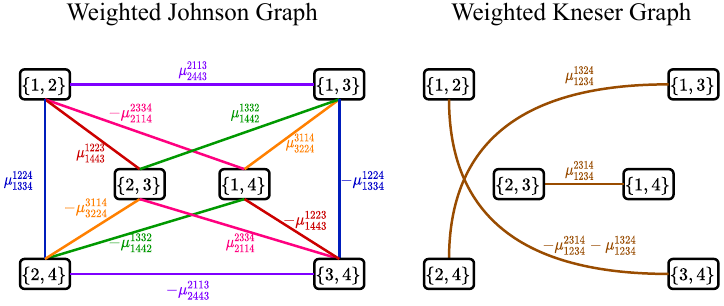}
			\caption{Johnson and Kneser graphs for $n=4$, where each edge weight is related to some dual variables from the SoS constraints.
			Edge weights with the same color sum to zero, see eq.\eqref{eq:kneser_constraints}.}
			\label{fig:johnson_kneser_example}
		\end{figure}
		
		Let $\oplus$ denote the symmetric difference of sets.
		Also, let $\mW^{\gJ}$ and $\mW^{\gK}$ denote the weight matrices of the Johnson and Kneser graphs, respectively.
		Then, for any $n$, the third and fourth list of constraints of problem \eqref{eq:opt_sos} translate to having the following constraints on $\mW^{\gJ}$ and $\mW^{\gK}$,
		\begin{align}
			\sum_{\substack{\gC_1,\gC_2\\ \gC_1 \oplus \gC_2 = \{i,j\}}} \hspace{-0.15in}\mW^{\gJ}_{\gC_1,\gC_2}= 0, \ \forall\ i < j \in [n], 
			\sum_{\substack{\gC_1,\gC_2\\ \gC_1 \oplus \gC_2 = \{i,j,k,\ell\}}} \hspace{-0.15in} \mW^{\gK}_{\gC_1,\gC_2}= 0,\  \forall\ i < j < k < \ell \in [n].\label{eq:kneser_constraints}
		\end{align}
		Thus, by using the construction in eq.\eqref{eq:V_setting}, we have that the PSD constraint of the dual formulation \eqref{eq:dual_problem} can be rewritten in terms of $\mW^{\gJ}$ and $\mW^{\gK}$ as follows,
		\begin{align}
			\mLambdat = \frac{\diag\big( \mXt \mYt \big)}{n-2} - \frac{\mXt}{n-2} + \left(\diag(\mW^{\gJ} \mYt) - \mW^{\gJ}\right) + \left(\diag(\mW^{\gK} \mYt) - \mW^{\gK}\right) \succeq 0. \notag
		\end{align}
		Let $\gGt = \Gt \cup \gJ \cup \gK$ such that $\mW^{\gGt} = \frac{1-2p}{n-2}\mW^{\Gt} + \mW^{\gJ} + \mW^{\gK}$, and noting that w.l.o.g. one can multiply the weights in eq.\eqref{eq:kneser_constraints} by $\ys_i \ys_j$ and $\ys_i\ys_j\ys_k\ys_\ell$, respectively.
		We can use a similar argument to that of eq.\eqref{eq:similar_matrix} and obtain
		\begin{align}
			\lambda_2(\E[\mLambdat]) = \lambda_2(\mL^{\gGt}). \label{eq:algconn_final}
		\end{align}
		The subtlety for lower bounding eq.\eqref{eq:algconn_final} is that, unless all edge weights are zero, the Johnson and Kneser graphs will both have at least one negative edge weight in order to fulfill  eq.\eqref{eq:kneser_constraints}.
		In other words, the Laplacian matrix $\mL^{\gGt}$ is no longer guaranteed to be PSD.
		That fact alone rules out almost all existing results on lower bounding the algebraic connectivity as it is mostly assumed that all edge weights are positive.
		Among the few works that study the Laplacian matrix with negative weights, one can find \cite{zelazo2014definiteness,chen2016definiteness}; however, their results focus on finding conditions for positive semidefiniteness of the Laplacian matrix in the context of electrical circuits and not in finding a lower bound.
		Our next result, generalizes the lower bound in \citep{mohar1991laplacian} by considering negative edge weights.
		\begin{theorem}
		\label{thrm:cheeger_ineq}
			Let $H = H^+ \cup H^-$ be a weighted graph such that $H^+$ and $H^-$ denote the disjoint subgraphs of $H$ with positive and negative weights, respectively.
			Also, let $\deg^{H^+}_{\max}$ denote the maximum node degree of $H^+$.
			Then, we have that
			\(
				\lambda_2(\mL^H) \geq \frac{\phi(H^{+})^2}{2 \deg^{H^+}_{\max}} + 2\cdot \mathrm{mincut}(H^{-}).
			\)
		\end{theorem}
		In Appendix \ref{app:discussion_lb}, we provide further discussion about Theorem \ref{thrm:cheeger_ineq}.
		In the case when there are positive weights only, the theorem above yields the typical Cheeger bound \citep{mohar1991laplacian}.
		When there is at least one negative weight, the bound shows an interesting trade-off between the Cheeger constant of the positive subgraph and the minimum cut of the negative subgraph.
		By applying Theorem \ref{thrm:cheeger_ineq} to eq.\eqref{eq:algconn_final}, we obtain
		\begin{align}
		\label{eq:final_lb}
			\lambda_2(\E[\mLambdat]) \geq \phi(\gGt^{+})^2 / (2 \deg^{\gGt^+}_{\max}) + 2\cdot \mathrm{mincut}(\gGt^{-}).
		\end{align}
		Without the weights of the Johnson and Kneser graphs, the lower bound above is equal to that of eq.\eqref{eq:lb_mu_zero}.
		Also, recall that, by construction, the edge set of the level-2 graph $\Gt$ is a subset of the edge set of the Johnson graph, and that the Kneser graph is the complement of the Johnson graph.
		That means that $\gGt$ will  be a complete graph of ${n \choose 2}$ vertices, where the edge weights of the Kneser graph are exclusively related to the dual variables $\mu$, while the edge weights of the Johnson graph might have an interaction between the noisy edge observations and the dual variables $\mu$.
		From the concentration argument stated after eq.\eqref{eq:exp_plus_concentration}, we conclude that as the lower bound in eq.\eqref{eq:final_lb} increases then the more likely to realize exact recovery.
		Next, for further clarity, we provide a detailed example of our analysis in this section.

	\section{Example}
\label{sec:example}
	The goal of this section is to provide a concrete example where the SoS relaxation \eqref{eq:opt_sos_degree_4} \textit{achieves} exact recovery but the SDP relaxation \eqref{eq:opt_sdp} \textit{does not}.
	Since for any input graph with $n$ vertices, its level-2 version has ${n \choose 2}$ vertices, we select a value of $n=5$ so that the level-2 graph has $10$ nodes and the plots can still be visually inspected in detail.  
	\begin{figure*}[!tb]
			\centering	
			\includegraphics[width=0.9\linewidth]{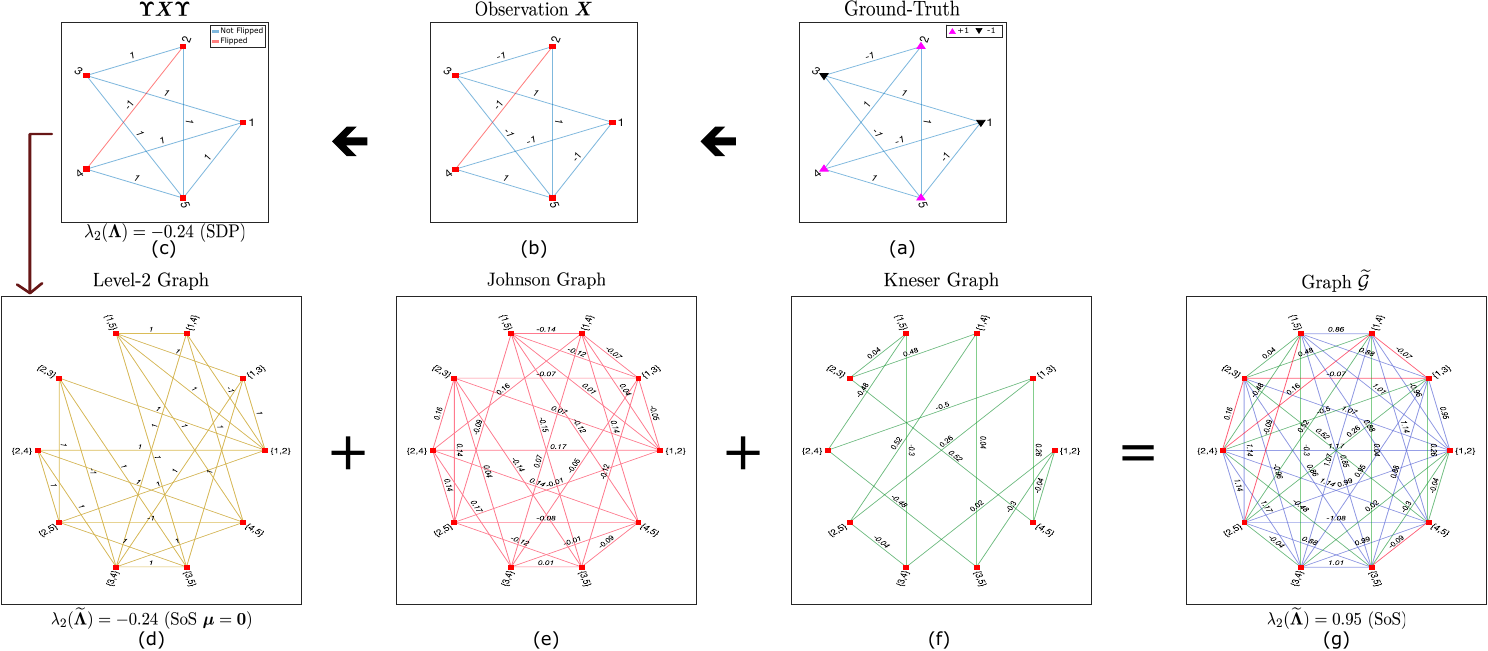}
			\caption{Detailed example of how the level-2 SoS relaxation results in improving the algebraic connectivity of the input graph through a combination of weights of its level-2 version, and the Johnson and Kneser graphs. 
			In the final graph $\gGt$, green and red lines indicate that their weights remain unchanged w.r.t. the Kneser and Johnson edge weights, respectively; while blue lines indicate that their weights resulted from the summation of weights from the Level-2 and Johnson graphs.}
			\label{fig:final_example}
		\end{figure*}	
	Figure (\ref{fig:final_example}a) shows the ground-truth labels of a graph with 5 nodes and 8 edges.
	Figure (\ref{fig:final_example}b) corresponds to the observation matrix $\mX$.
	In this case, only one edge is corrupted (the red edge).
	Figure (\ref{fig:final_example}c) shows the graph where an edge label of $-1$ or $1$ indicates whether the observed edge value was corrupted or not, respectively.
	The latter graph is obtained by $\mUpsilon \mX \mUpsilon$, where $\mUpsilon$ denotes a diagonal matrix with entries from $\vys$, similar to the procedure in eq.\eqref{eq:upsilon}. 
	Let $\mLambda$ be the dual variable of the PSD constraint in the SDP relaxation \eqref{eq:opt_sdp}.
	Then, under a similar dual construction to the one in \citep{bello2019exact,abbe2016exact}, we have that $\lambda_2(\mLambda) = \min_{\vv \perp \vys} \frac{\vv^\top \mLambda \vv}{\vv^\top \vv}$ is equal to the second eigenvalue of the Laplacian matrix of Figure (\ref{fig:final_example}c). 
	 Thus, we can observe that, for SDP, the ground-truth solution $\mYs$ attains a value of $\lambda_2(\mLambda) = -0.24 < 0$, hence, exact recovery fails.
	 
	In Figure (\ref{fig:final_example}d), we show the level-2 graph of $\mUpsilon \mX \mUpsilon$, i.e., $\widetilde{\mUpsilon} \mXt \widetilde{\mUpsilon}$.
	As argued by Proposition \ref{prop:sdp_sos_equal}, by setting $\vmu = \vzero$ the SoS does not do any better than SDP, which is verified by obtaining $\lambda_2(\mLambdat) = -0.24 < 0$, hence, exact recovery also fails in this case.
	However, by solving problem \eqref{eq:opt_sos},
%	\footnote{Any general purpose SDP solver can be used, e.g., CVX \citep{cvx}.}
	we obtain $\vmu \neq \vzero$ which, as discussed in Section \ref{sec:johnson_kneser}, relates to edge weights in the Johnson and Kneser graphs.
	Those edge weights are depicted in Figures (\ref{fig:final_example}e) and (\ref{fig:final_example}f), respectively.
	Finally, after summing all the weights of the level-2 graph, Johnson and Kneser graphs, we obtain a complete graph depicted in Figure (\ref{fig:final_example}g).
	In the latter, we have that $\lambda_2(\mLambdat) = 0.95 > 0$, which guarantees that $\mYt = \vyt {\vyt}^\top$, i.e., exact recovery succeeds.

	Motivated by eq.\eqref{eq:kneser_constraints} and Theorem \ref{thrm:cheeger_ineq}, in Appendix \ref{app:construction}, we show a non-trivial construction of the Kneser graph weights based on only the node degrees of the level-2 graph.

\section{Concluding Remarks}
	We studied the statistical problem of exact recovery in graphs over the boolean hypercube.
	We considered a generative model similar to that of \citep{globerson2015hard} and thoroughly analyzed the level-2 SoS relaxation \eqref{eq:opt_sos} of problem \eqref{eq:opt_dis}, motivated by empirical evidence on improvement in exact recoverability over the SDP relaxation \eqref{eq:opt_sdp}.
	We showed how the dual formulation of the SoS relaxation relates to finding edge weights for the Johnson and Kneser graphs so that the algebraic connectivity of the input graph increases.
	Finally, we characterized the improvement by deriving a novel lower bound for the algebraic connectivity of graphs with positive and negative weights, and provided a construction of the Kneser graph weights in Appendix \ref{app:construction}.
	It remains an interesting future work to study the recent mapping of degree-2 to degree-4 solutions in \cite{mohanty2020lifting} for exact recovery.

	\bibliographystyle{agsm}
	\bibliography{sos_inference}

	\appendix	
	\onecolumn
\def\toptitlebar{
	\hrule height4pt
	\vskip .25in}

\def\bottomtitlebar{
	\vskip .25in
	\hrule height1pt
	\vskip .25in}

\thispagestyle{empty}
\hsize\textwidth
\linewidth\hsize \toptitlebar {\centering
{\large\bf SUPPLEMENTARY MATERIAL \\ A Thorough View of Exact Inference in Graphs from the Degree-4 SoS Hierarchy \par}}
\vspace{-0.1in} \bottomtitlebar

\section{Detailed Proofs}
\label{app:proofs}

In this section, we state the proofs of all propositions and theorem. 

\subsection{Proof of Proposition \ref{prop:equivalent_form}}
\label{proof1}
By construction of the level-2 matrix $\mXt$, we have that each entry $X_{i,j}$ is repeated $n-2$ times.
Thus, it follows that the objectives in problems \eqref{eq:opt_sos_degree_4} and \eqref{eq:opt_sos} are equal.

Let $\mY$ be a feasible solution to problem \eqref{eq:opt_sos_degree_4}, then clearly the principal submatrix  indexed by ${[n] \choose 2} \times {[n] \choose 2}$ is a feasible solution to problem \eqref{eq:opt_sos}.
It remains to verify that if $\mYt$ is a feasible solution to problem \eqref{eq:opt_sos} then there exists a matrix $\mY$ such that it is feasible to problem \eqref{eq:opt_sos_degree_4} and has $\mYt$ as a principal submatrix.
We define the entries of $\mY$ as follows,
\begin{align}
	Y_{(i,i)(j,j)} = \Yt_{\{i,j\},\{i,j\}} \notag \\
	Y_{(i,i)(j,k)} = \Yt_{\{i,j\},\{i,k\}} \notag \\
	Y_{(i,j)(k,\ell)} = \Yt_{\{i,j\},\{k,\ell\}}. \notag
\end{align}
Clearly, $\mY$ will fulfill the constraints of problem \eqref{eq:opt_sos_degree_4} if $\mYt$ is feasible to problem \eqref{eq:opt_sos}. 
In particular, one can verify that $\vv^\top \mY \vv \geq 0$ for any $\vv$ if $\mYt \succeq 0$, which concludes our proof.

\subsection{Proof of Proposition \ref{prop:sdp_sos_equal}}
We will show the equivalence between problem \eqref{eq:opt_sos_degree_4}, \textit{without the third and fourth list of constraints}, and problem \eqref{eq:opt_sdp}.
Then, by Proposition \ref{prop:equivalent_form}, our claim follows.

The proof is similar to that of Section \ref{proof1}, where it is clear that the objectives in problems \eqref{eq:opt_sdp} and \eqref{eq:opt_sos_degree_4} are equal.
Let $\mY^\sdp$ be a feasible solution to problem \eqref{eq:opt_sdp}, then we define $\mY^\sos$ as follows,
\begin{align*}
	Y^\sos_{(i,i)(j,k)} &= Y^\sdp_{i,j} \\
	Y^\sos_{(i,j)(k,\ell)} &= 0.
\end{align*}
Since $\mY^\sdp \succeq 0$, it follows that $\mY^\sos \succeq 0$ and, thus, $\mY^\sos$ is feasible to problem \eqref{eq:opt_sos_degree_4} without the third and fourth list of constraints.
Similarly, in the other direction, let $\mY^\sos$ be a feasible solution to problem \eqref{eq:opt_sos_degree_4} without the third and fourth list of constraints, and define $\mY^\sdp$ to be the principal submatrix of $\mY^\sos$ with the first $n$ rows and columns. 
Then, it follows that if $\mY^\sos \succeq 0$ then $\mY^\sdp \succeq 0$, which is feasible to problem \eqref{eq:opt_sdp}.

\subsection{Proof of Theorem \ref{thrm:cheeger_ineq}}

For simplicity, let $\mW$ and $\mL$ be the weight matrix and Laplacian matrix of an undirected connected graph $H$ of $m$ nodes.
Also, let $\mW^+$ and $\mW^-$ be the weight matrices of $H^+$ and $H^-$.
For a matrix $\mM$ and vector $\vv$, we use $R_\mM(\vv)$ to denote their Rayleigh quotient, i.e., $R_\mM(\vv) = \frac{\vv^\top \mM \vv}{\vv^\top \vv}$.
It follows that $R_\mL(\vv) := \frac{\vv^\top \mL \vv}{\vv^\top \vv} = \frac{\sum_{i<j} W_{i,j} (v_i - v_j )^2}{\vv^\top \vv} $, and $\lambda_2(\mL) = \min_{\vv \perp \vone} R_\mL(\vv)$. 
Similarly, we define $R_\mL^+ (\vv) := \frac{\sum_{i<j} W_{i,j}^+ (v_i - v_j )^2}{\vv^\top \vv}$, $R_\mL^- (\vv) := \frac{\sum_{i<j} W_{i,j}^- (v_i - v_j )^2}{\vv^\top \vv}$. 
Note that $R_\mL(\vv) = R_\mL^+ (\vv) + R_\mL^- (\vv)$.
Next, we state a lemma that will be of use for the proof of Theorem \ref{thrm:cheeger_ineq}.
\begin{lemma}
\label{lemma:RLshifting}
	Let $\mL$ be a Laplacian matrix of dimension $m \times m$.
	Let also $\vone$ denote a vector of ones. 
	Then, for any $\delta \in \sR, \vv \in\sR^m, \sum_i v_i \geq 0$, it follows that
	\[
		R_{\mL}^+ (\vv)  \geq  R_{\mL}^+ ( \vv + \delta \vone) \,.
	\]
\end{lemma}
\begin{proof}
	Starting from the right-hand side, we have
	\begin{align*}
	R_{\mL}^+ ( \vv + \delta \vone)
	&= \frac{\sum_{i<j} W_{i,j}^+ \left(( v_i + \delta)-( v_j + \delta)\right)^2}{\sum_{i}  \left( v_i + \delta \right)^2} \\
	&= \frac{ \sum_{i<j} W_{i,j}^+ \left(v_i - v_j \right)^2}{\sum_{i}  \left( v_i + \delta \right)^2} \\
	&= \frac{ \sum_{i<j} W_{i,j}^+ \left(v_i - v_j \right)^2}{\sum_{i}  \left(  v_i^2 + \delta^2 + 2\delta v_i \right)} \\
	&= \frac{ \sum_{i<j} W_{i,j}^+ \left(v_i - v_j \right)^2}{ \sum_{i} v_i^2 + m\delta^2 + 2\delta \sum_i v_i } \\
	&\overset{(a)}{\leq} \frac{ \sum_{i<j} W_{i,j}^+ \left(v_i - v_j \right)^2}{ \sum_{i} v_i^2 + m\delta^2} \\
	&\leq \frac{ \sum_{i<j} W_{i,j}^+ \left(v_i - v_j \right)^2}{ \sum_{i} v_i^2} \\
	&= R_{\mL}^+ (\vv),
	\end{align*}
	where (a) holds by the fact that $\sum_i v_i \geq 0$.
\end{proof}
We now present the proof of Theorem \ref{thrm:cheeger_ineq}.
\begin{proof}
	Let $\vv$ be the eigenvector related to the eigenvalue $\lambda_2(\mL)$. 
	Without loss of generality, we assume $\norm{\vv} = 1$ and $v_1 \leq v_2 \leq \dots \leq v_m$. 
	Recall that $\vone^\top \vv = 0$.
	Then, we have that
	\[
	\lambda_2(\mL) = R_\mL(\vv) = R_\mL^+(\vv) + R_\mL^-(\vv). %\geq R_\mL^+(\vu) + R_\mL^-(\vu+\delta \vone).
	\]
	
	\paragraph{Lower bounding $R_\mL^+(\vv)$.}
	Set $\delta = v_1$ and denote $\vu = \vv - \delta\vone$.
	Then, we have that $0 = u_1 \leq \dots \leq u_m$. 
	Also note that $\delta^2 \leq 1$. 
	Then, by Lemma \ref{lemma:RLshifting}, it follows that $R_\mL^+(\vv) \geq R_\mL^+(\vu)$.
	
	We now define a random variable $t$ on the support $[0, u_m]$, with probability density function $f(t) = \frac{2}{u_m^2}t$. One can verify that $\int_{t = 0}^{u_m} \frac{2}{u_m^2}t \ dt = 1$, thus $f(t)$ is a valid probability density function.
	Then, for any interval $[a,b]$, it follows that the probability of $t$ falling in the interval is 
	\[
	\sP[a \leq t \leq b] = \int_{t=a}^b \frac{2}{u_m^2}t \ dt = \frac{1}{u_m^2} (b^2 - a^2).
	\]
	Next, for some $t$, construct a random set $S_t = \{i \mid u_i \geq t\}$. 
	Let $\omega^+ (\partial S_t) = \sum_{i\in S_t, j\notin S_t} W_{i,j}^+$.
	
	It follows that 
	\begin{align*}
	\E[w^+ (\partial S_t)]
		&= \E[\sum_{i\in S_t, j\notin S_t} W_{i,j}^+] \\
		&= \sum_{i<j} \sP[u_j \leq t \leq u_i] W_{i,j}^+ \\
		&= \frac{1}{u_m^2} \sum_{i<j} (u_i - u_j) (u_i + u_j)   W_{i,j}^+ \\
		&\leq \frac{1}{u_m^2} \sqrt{\sum_{i<j} (u_i-u_j)^2 W_{i,j}^+} 
		\sqrt{\sum_{i<j}(u_i + u_j)^2 W_{i,j}^+} \\
		&= \frac{1}{u_m^2}\sqrt{R_\mL^+(\vu) \sum_i u_i^2} 
		\sqrt{\sum_{i<j}(u_i + u_j)^2 W_{i,j}^+} \\
		&\leq \frac{1}{u_m^2}\sqrt{R_\mL^+(\vu) \sum_i u_i^2} 
		 \sqrt{2\sum_{i} u_i^2 \deg^{H^+}(i)}  \\
		 &\leq \frac{1}{u_m^2}\sqrt{R_\mL^+(\vu) \sum_i u_i^2} 
		 \sqrt{2\deg^{H^+}_{\max} \sum_{i} u_i^2}  \\
		&= \frac{\sum_{i} u_i^2}{u_m^2} \sqrt{2 \deg^{H^+}_{\max} R_\mL^+(\vu)}.
	\end{align*}
	Also note that 
	$\E\left[\AAbs{S_t}\right]
	= \sum_i \sP[u_i \geq t] 
	= \sum_i \frac{u_i^2}{u_m^2}.$
	As a result we obtain 
	$$
	\E [\omega^+ (\partial S_t)]
		\leq 
		\E \left[ \AAbs{S_t} \right] \sqrt{2 \deg^{H^+}_{\max} R_\mL^+(\vu)}.
	$$
	Thus, we have $\E \left[\omega^+ (\partial S_t) - \AAbs{S_t} \sqrt{2 \deg^{H^+}_{\max} R_\mL^+(\vu)} \right] \leq 0.$
	This implies that $\exists S_t$ such that $\omega^+ (\partial S_t) - \AAbs{S_t} \sqrt{2 \deg^{H^+}_{\max} R_\mL^+(\vu)} \leq 0$. 
	Rearranging we have, 
	\begin{align}
	\label{eq:RLplus}
		 R_\mL^+(\vv) \geq R_\mL^+(\vu) \geq \frac{\omega^+(\partial S_t)^2}{2\deg^{H^+}_{\max} \AAbs{S_t}^2}
	\end{align}

	\paragraph{Lower bounding $R_\mL^-(\vv)$.}
	Set $\alpha = \sqrt{\frac{1}{v_1^2 + v_m^2}}$ and denote $\vu = \alpha\vv$.
	Then, we have that $u_1^2 + u_m^2 = 1$. 
	Note also that $R_\mL^-(\vv) = R_\mL^-(\vu)$.
	
	We now define a random variable $t$ on the support $[u_1, u_m]$, with probability density function $f(t) = 2 |t|$. 
	One can verify that $\int_{t = u_1}^{u_m} 2|t| \ dt = 1$, thus $f(t)$ is a valid probability density function.
	Then, for any interval $[a,b]$, it follows that the probability of $t$ falling in the interval is 
	\[
	\sP[a \leq t \leq b] = \int_{a}^b 2|t| \ dt = b^2 \sign(b) - a^2 \sign(a).
	\]
	Since $[u_1,u_m] \subset [-1,1]$, one can verify that $(a-b)^2/2  \leq \sP[a \leq t \leq b]$.
	Let $\omega^- (\partial S_t) = \sum_{i\in S_t, j\notin S_t} W_{i,j}^-$.
	For some $t$, construct a random set $S_t = \{i \mid u_i \leq t\}$. 
	It follows that 
	
	\begin{align}
	\E[\omega^- (\partial S_t)]
	&= \E[\sum_{i\in S_t, j\notin S_t} W_{i,j}^-] \nonumber\\
	&= \sum_{i<j} \sP[u_i \leq t \leq u_j] W_{i,j}^- \nonumber\\
	&\leq \frac{1}{2} \sum_{i<j} (u_i - u_j)^2    W_{i,j}^- \nonumber\\
	&= \frac{1}{2} R_\mL^-(\vu) \sum_i u_i^2 \nonumber\\
	&\leq \frac{1}{2} R_\mL^-(\vu),	\notag
	\end{align}
	where the last inequality follows from having $\sum_i u_i^2 \geq 1$ and $R_\mL^-(\vu) \leq 0$.
	Thus, we have $\E[\omega^- (\partial S_t) - \frac{1}{2} R_\mL^-(\vu)] \leq 0$.
	This implies that $\exists S_t$ such that $\omega^- (\partial S_t) - \frac{1}{2} R_\mL^-(\vu) \leq 0$.
	Rearranging we have,
	\begin{align}
	\label{eq:RLminus}
		 R_\mL^-(\vv) = R_\mL^-(\vu) \geq 2 \omega^-(\partial S_t).
	\end{align}
	By minimizing \eqref{eq:RLplus} and \eqref{eq:RLminus} independently, and combining them, we conclude our proof.
\end{proof}

\newpage
\section{Further Discussion on Theorem \ref{thrm:cheeger_ineq}}
\label{app:discussion_lb}

We remark that the reason we do not consider other versions of the Laplacian matrix (e.g., the normalized Laplacian matrix which is guaranteed to be PSD even in the presence of negative weights) is because how our primal/dual construction (see Section \ref{sec:dual_problem}) leads to a valid solution of the constraints in eq.\eqref{eq:dual_problem} and that also satisfies the KKT conditions.
That is, using other notions of Laplacian matrix (see e.g., \cite{kunegis2010spectral,mercado2016clustering,cucuringu2019sponge,chiang2012scalable,knyazev2017signed,atay2020cheeger}) would \textbf{not} satisfy the optimality conditions needed for exact recovery (in particular, stationarity and complementary slackness).
In fact, one of the challenges we face in our analysis is that by having the standard Laplacian matrix, its minimum eigenvalue can be negative, as shown in our example in Section \ref{sec:example} and also discussed in \cite{knyazev2017signed}, which motivated the search of a more general lower bound for the algebraic connectivity of signed graphs (Theorem \ref{thrm:cheeger_ineq}).

We also highlight that Theorem \ref{thrm:cheeger_ineq} sheds light on the subtle trade-off between the Cheeger constant of the positive subgraph and the minimum cut of the negative subgraph.
That is, intuitively, the SoS solution will try to find negative weights for the Johnson and Kneser graphs of as low magnitude as possible, so that the minimum-cut of the negative subgraph does not make the algebraic connectivity negative.

\section{A Degree-Based Construction of the Kneser Graph}
\label{app:construction}

	In Section \ref{sec:example}, we used CVX \citep{cvx} to solve problem \eqref{eq:opt_sos} and, thus, obtain the dual variables $\vmu$ in problem \eqref{eq:dual_problem} from which we construct the weights of the Johnson and Kneser graphs.
	Motivated by the trade-off between Cheeger constants of the positive and negative subgraphs, shown in Theorem \ref{thrm:cheeger_ineq}, we show a simple non-trivial way (not necessarily optimal) to directly construct the weights of the Kneser graph.
	The reason why we focus in the Kneser graph weights is because the fourth list of constraints in problem \eqref{eq:opt_sos} can be expressed by two constraints for any $i < j < k < \ell$, as noted in Section \ref{sec:johnson_kneser}.
	The latter fact implies that, for any $i < j < k < \ell$, the edge weights $W^\gK_{\{i,j\},\{k,\ell\}}$, $W^\gK_{\{i,k\},\{j,\ell\}}$, and $W^\gK_{\{i,\ell\},\{j,k\}}$ need to sum to zero in order to fulfill the SoS constraints.
	As also noted in Section \ref{sec:johnson_kneser}, at least one of the previous weights need to be negative unless all three are zero.
	With these considerations, we present our construction in Algorithm \ref{algo:construction}, which relies only on the node degrees and a constant real value.
	\renewcommand{\algorithmicrequire}{\textbf{Input:}}
	\renewcommand{\algorithmicensure}{\textbf{Output:}}
	\begin{algorithm}[!t]
	\caption{A construction of Kneser graph weights}
	\label{algo:construction}
		\begin{algorithmic}[1]
			\REQUIRE Level-2 weight matrix $\mM = \mUpsilont \mXt \mUpsilont$, constant $c \in \sR$.
			\STATE $\deg(\gC_1) \gets \sum_{\gC_2} M_{\gC_1, \gC_2}, \ \forall \gC_1 \in {[n] \choose 2}$
			\STATE Initialize $\mW^\gK$ as a zero matrix
 			\FORALL{$i < j < k < \ell \in [n]$}
				\STATE Assign the following such that $\psi_1 \geq \psi_2 \geq \psi_3$
				\STATE $\psi_1 \gets \deg(\{i,j\}) + \deg(\{k,\ell\})$
				\STATE $\psi_2 \gets \deg(\{i,k\}) + \deg(\{j,\ell\})$
				\STATE $\psi_3 \gets \deg(\{i,\ell\}) + \deg(\{j,k\})$
				\IF{$\psi_1 = \psi_2 = \psi_3$}
					\STATE $W^\gK_{\{i,j\},\{k,\ell\}} \gets 0,\ \ W^\gK_{\{i,k\},\{j,\ell\}} \gets 0, \ \ W^\gK_{\{i,\ell\},\{j,k\}} \gets 0$
				\ELSIF{$\psi_1 = \psi_2$}
					\STATE $W^\gK_{\{i,j\},\{k,\ell\}} \gets -c, \ \ W^\gK_{\{i,k\},\{j,\ell\}} \gets -c, \ \ W^\gK_{\{i,\ell\},\{j,k\}} \gets 2c$
				\ELSE
					\STATE $W^\gK_{\{i,j\},\{k,\ell\}} \gets -2c,\ \ W^\gK_{\{i,k\},\{j,\ell\}} \gets c, \ \ W^\gK_{\{i,\ell\},\{j,k\}} \gets c$
				\ENDIF
			\ENDFOR
			\STATE $\mW^\gK \gets \mW^\gK + (\mW^\gK)^\top$ \COMMENT{To symmetrize.}
			\ENSURE $\mW^{\gGt} \gets \mM + \mW^\gK$
		\end{algorithmic}
	\end{algorithm}
	
	The intuition behind Algorithm \ref{algo:construction} is that the negative weight will be assigned to the edge that connects the two nodes that have the highest combined node degree.
	In Lines 8-9, if all three edges have the same combined node degree then we set all three weights to zero.
	In Lines 10-11, the edge with lowest combined node degree is set to $2c$, while the other edges that attain the same combined node degree are set to $-c$.
	In Line 13, the edge with highest combined node degree is set to $-2c$, while the other edges  are set to $c$.
	It is clear that the SoS constraints will be fulfilled for each quadruple $i < j< k< \ell$.
	Finally, we note that if $c=0$ then simply the same input, $\mM$, is returned.
	The latter implies that, for the \textit{optimal} value of $c$, Algorithm \ref{algo:construction} cannot return a weight matrix with lower algebraic connectivity than that of $\mM$.\footnote{Recall that the SDP problem \eqref{eq:opt_sdp} would attain an algebraic connectivity equal to that of $\mM$ if the optimal solution is $\vys \vys^\top$.}
	
	Recall that $\lambda_2(\mLambdat) = \lambda_2(\mL^{\gGt})$.
	In Figure \ref{fig:construction}, we ran Algorithm \ref{algo:construction} with input graph $\mUpsilont \mXt \mUpsilont$ equal to the graph in Figure (\ref{fig:final_example}d), and $c \in [0, 0.6]$.
	For each $c$, we plotted the algebraic connectivity of our construction.
	We observe that when $c=0$, in effect $\lambda_2(\mLambdat) = -0.24$ as pointed in Figure (\ref{fig:final_example}d).
	In this example, the optimal value of $c$ is $0.32$ and attains a $\lambda_2(\mLambdat)$ of $0.9368$, which is very close to the value $\lambda_2(\mLambdat) = 0.95$ found by CVX (see Figure (\ref{fig:final_example}g)).
	Finally, we also plot the Kneser graph weights for $c=0.32$ following the construction in Algorithm \ref{algo:construction}.
	\begin{figure}[!t]
		\centering	
		\includegraphics[width=0.49\linewidth]{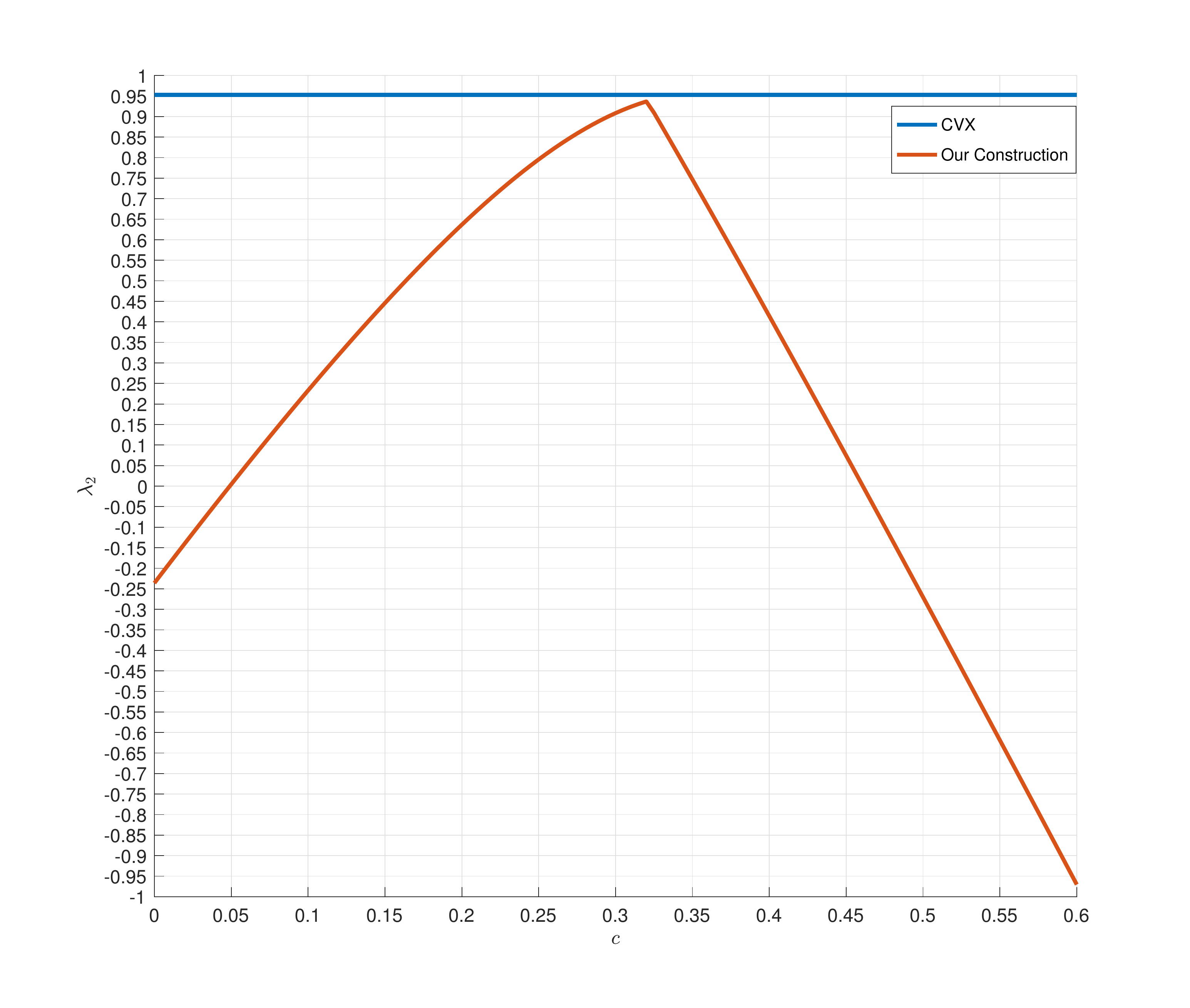}
		\hfill
		\includegraphics[width=0.49\linewidth]{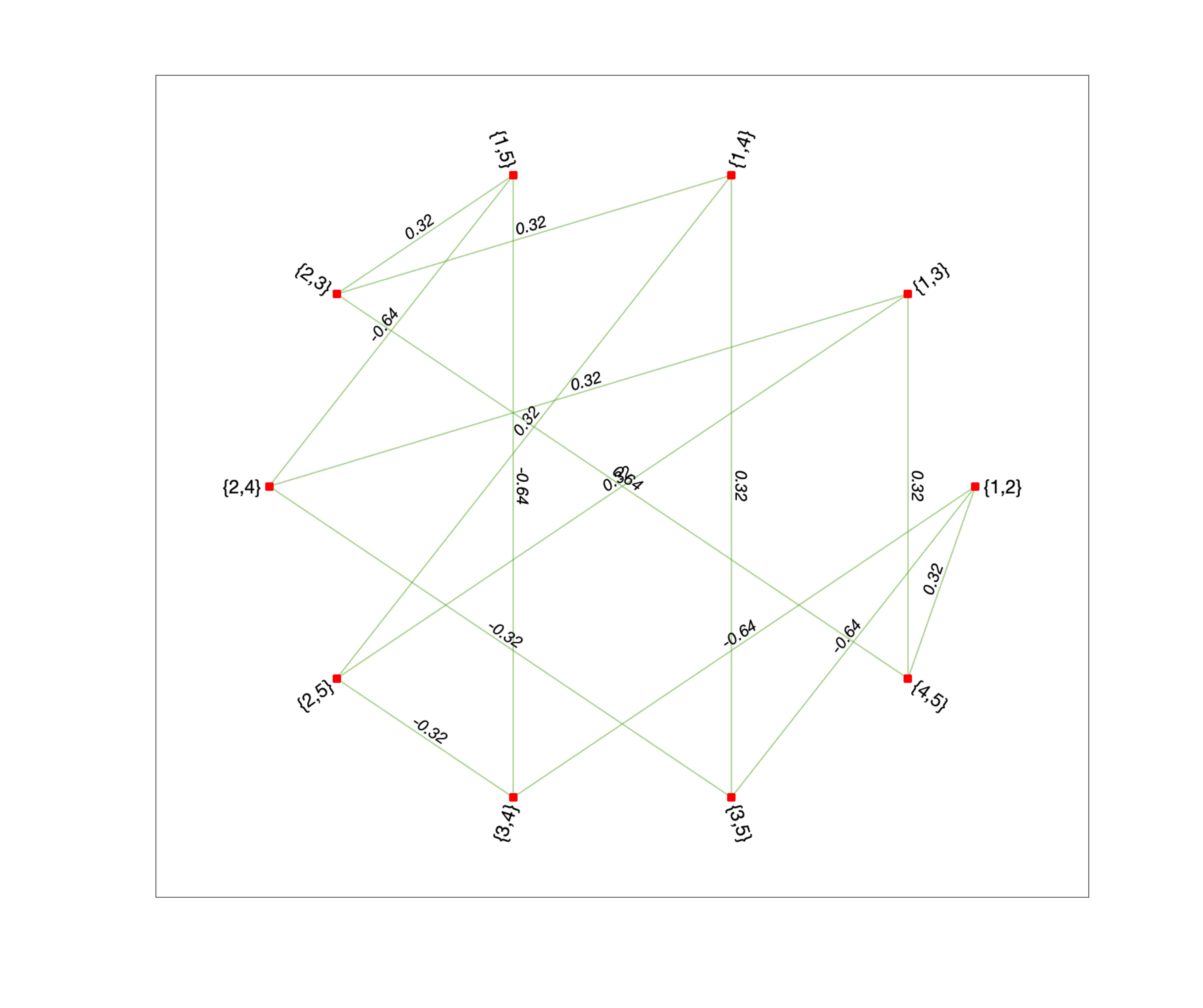}
		\caption{(Left) The blue line is the algebraic connectivity found by CVX, i.e., $0.95$ as pointed in Figure (\ref{fig:final_example}g).
		The red line is the algebraic connectivity of our construction in Algorithm \ref{algo:construction} for different values of $c \in [0, 0.6]$.
		(Right) The Kneser graph weights for the optimal $c= 0.32$, which in effect differs from the weights found by CVX in Figure (\ref{fig:final_example}f).
		 }
		\label{fig:construction}
	\end{figure}

\end{document}